\documentclass{article}

\PassOptionsToPackage{numbers, compress}{natbib}

\usepackage[preprint]{neurips_2026}

\usepackage[utf8]{inputenc} 
\usepackage[T1]{fontenc}    
\usepackage{hyperref}       
\usepackage{url}            
\usepackage{booktabs}       
\usepackage{amsfonts}       
\usepackage{nicefrac}       
\usepackage{tabularx}
\usepackage{microtype}      
\usepackage{floatrow}
\usepackage{caption}

\usepackage[table]{xcolor}
\usepackage{booktabs}
\usepackage{makecell}
\usepackage{xspace}
\usepackage{subcaption}
\usepackage{algorithm}
\usepackage{algpseudocode}
\usepackage{threeparttable}
\usepackage{wrapfig}
\definecolor{lightyellow}{RGB}{255, 248, 220}
\usepackage{amsmath, amssymb, amsthm, amsfonts}
\usepackage{times}
\usepackage{graphicx}
\usepackage{enumitem}
\usepackage{cleveref}
\usepackage{titlesec}
\usepackage{mathrsfs}
\usepackage{physics}
\usepackage{booktabs}
\usepackage{tensor}
\usepackage{nicematrix}
\usepackage{thm-restate}

\titlespacing*{\section}{0pt}{6pt}{0pt}
\titlespacing*{\subsection}{0pt}{6pt}{0pt}

\usepackage{bm}

\newcommand{\W}{\bm{W}}
\newcommand{\bZ}{\bm{Z}}

\newcommand{\X}{\bm{X}}

\newcommand{\sg}{\mathsf{stopgrad}}
\newcommand{\down}{\mathrm{down}}

\newtheorem{proposition}{Proposition}
\newtheorem{assumption}{Assumption}

\theoremstyle{remark}
\newtheorem{remark}{Remark}

\ifdefined\usebigfont

\usepackage{times}
\usepackage[fontsize=13pt]{scrextend}
\AtBeginDocument{
	\newgeometry{letterpaper,left=1.56in,right=1.56in,top=1.71in,bottom=1.77in}
}
\else
\fi

\usepackage{makecell}
\definecolor{lightyellow}{RGB}{255, 248, 220}

\usepackage{booktabs,graphicx,multirow}
\usepackage{tikz}
\usetikzlibrary{positioning,arrows.meta,calc,bending,fit}
\definecolor{cFast}{HTML}{C0392B}
\definecolor{cSlow}{HTML}{4A4A4A}
\definecolor{cMute}{HTML}{7A7A7A}
\definecolor{cTok}{HTML}{EDEDED}
\definecolor{cTokHi}{HTML}{BFBFBF}

\newcommand{\contextstart}{c_1}
\newcommand{\contextend}{c_2}

\newcommand{\hgap}{c_{\rm gap}}
\newcommand{\zalign}{c_{\rm align}}
\newcommand{\querypos}{q_1}
\newcommand{\newquerypos}{q_2}
\newcommand{\lwindow}{w_{T}}
\newcommand{\swindow}{w_{S}}
\newcommand{\ttcd}{\textcolor{black}{TTCD}\xspace}
\newcommand{\ipttcd}{\textcolor{black}{IP-TTCD}\xspace}
\title{Learning What to Remember:\\ Test-Time Training via Context Distillation}

\author{%
Zixuan Wang$^{1,2}$\thanks{Equal Contribution. Code available at \url{https://github.com/dangxingyu/ttcd}.} \qquad Xingyu Dang$^1$\footnotemark[1] \qquad Rui-Jie Zhu$^{3}$\qquad Zixin Wen$^{4}$ \AND Hengyu Fu$^2$\qquad Wenhao Chai$^1$\qquad Jason D. Lee$^2$
\\
\\
$^1$Princeton University
\qquad
$^2$UC Berkeley \qquad
$^3$UC Santa Cruz \\
$^4$Carnegie Mellon University
}

\begin{document}

\maketitle

\begin{abstract}
Effective long-context modeling is not merely about retaining more of the past, but about preserving the information that may prove relevant later. Test-time training (TTT) is an appealing approach that performs online parameter updates for long-context modeling, yet existing TTT methods only optimize either reconstruction or online adaptation objectives without considering the future utility of retained information. In this work, we propose \textbf{T}est-\textbf{T}ime \textbf{C}ontext \textbf{D}istillation (\ttcd), a TTT framework that introduces a self-supervised objective for allocating limited memory capacity for future use. Specifically, \ttcd uses a long-window teacher to supervise the fast weights of a short-window student, where the hidden-state discrepancy between them offers a dense, self-supervised signal guiding the model to memorize the contextual information crucial for future token predictions. We focus on an in-place variant: In-Place TTCD (\ipttcd), which uses the existing MLP parameters as the fast weights. Experiments on long-context language modeling tasks show \ipttcd consistently outperforms DeltaNet, Gated DeltaNet, sliding-window attention, and TTT when pre-trained from scratch. Furthermore, \ipttcd allows pre-trained transformer models to adapt their parameters during inference through continual pre-training, gaining long-context capabilities with only a lightweight architectural augmentation. Our results position \ttcd as a step toward architectural continual learning.
\end{abstract}

\section{Introduction}

As language models are increasingly used for long-chain reasoning and agentic trajectories, long-context language modeling \citep{su2024roformer, dao2022flashattention, gao2025train} has become a central direction for scaling models beyond short, isolated inputs. Full attention \citep{vaswani2017attention, radford2018improving} preserves all previous context, enabling exact retrieval and utilization of the past information but incurring computation and memory costs that grow quadratically with context length.
For humans, in contrast, effective  `long-context' reasoning rarely requires the entire history:
when processing a codebase or a book, people decide which information is likely to matter for future predictions and compress the rest away. Therefore, efficient alternatives face a  harder problem than lossless retention: they must compress the past into a bounded state.

Modern recurrent architectures approach this challenge by replacing the growing key–value cache with a compact recurrent state \citep{schlag2021delta, gu2021efficiently, gu2023mamba, yang2023gated, guo2025log}, achieving constant cost per token. Recent variants, including DeltaNet~\citep{schlag2021delta}, Gated DeltaNet~\citep{yang2023gated} and KDA~\citep{team2025kimi}, 
have narrowed the gap to softmax attention while showing favorable scaling. Test-time training (TTT) \citep{sun2020test, sun2024learning} offers a 
complementary perspective by making a subset of the model's parameters 
trainable at inference time, so that the weights themselves act as 
contextual memory. Several recent methods \citep{zhang2025test, 
feng2026inplace, tandon2025end} demonstrate that this paradigm is 
beneficial and practical for long-context language modeling.

However, under a fixed-capacity bottleneck, these alternative models must effectively decide at every step what to write, preserve, or forget to achieve better performance in a longer context. Existing TTT objectives provide only partial solutions to this issue. Reconstruction-based objectives \cite{sun2024learning, behrouz2024titans, zhang2025test} ask the state to store all history uniformly, but do not prioritize the storage of predictive information compared to irrelevant detail. The next-token objectives \cite{feng2026inplace, tandon2025end} identify what helps the current prediction, but do not explicitly attribute that predictive computation to remote history rather than to the local context. Instead, we desire an objective that directly teaches which past information should be retained for potential future use. Therefore, we ask the following research question:
\begin{center}
    \textit{Can we design a TTT objective that compresses past history according to its utility for future predictions?}
\end{center}


We propose \ttcd (\textbf{T}est-\textbf{T}ime \textbf{C}ontext \textbf{D}istillation), a test-time training method based on context distillation \citep{snell2022learning, corallo2024finch, jha2024characterizing, eyuboglu2025cartridges, chari2025kv}. The context distillation objective captures the counterfactual discrepancy between two views of the same model with different context lengths. Specifically, we calculate the difference between \textit{a long-context self-teacher} with more remote history tokens and \textit{a short-window student} without the context when predicting the next token. The hidden-state difference indicate what the remote information contributes beyond the local context.
The self-supervised objective therefore enables \ttcd to update the fast weight, encouraging the limited parametric state to store distant information that the student would miss without the context, providing a dense representation-level signal of the essential past for potential future prediction.

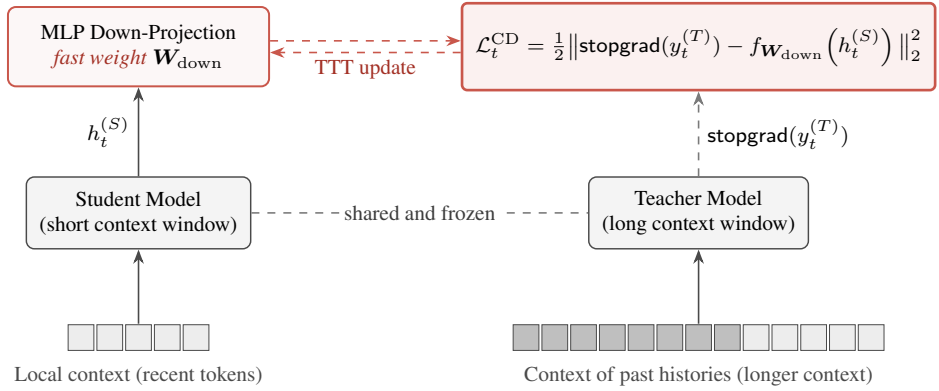
\begin{figure}[t]
    \centering
    \begin{tikzpicture}[
        font=\footnotesize,
        encbox/.style={rounded corners=3pt, draw=cSlow, line width=0.55pt, fill=gray!9, align=center, inner sep=5pt, minimum height=0.95cm, minimum width=2.7cm},
        predbox/.style={rounded corners=3pt, draw=cFast!85, line width=0.75pt, fill=cFast!4, align=center, inner sep=5pt, minimum height=1.05cm, minimum width=3.45cm},
        lossbox/.style={rounded corners=2pt, draw=cFast!85, line width=1.0pt, fill=cFast!7, align=center, inner sep=5pt, minimum height=1.15cm, minimum width=5.2cm},
        tokhi/.style={draw=cSlow, line width=0.35pt, fill=cTok, minimum size=0.34cm, inner sep=0pt},
        tokbox/.style={draw=cSlow, line width=0.35pt, fill=cTokHi, minimum size=0.34cm, inner sep=0pt},
        arr/.style={-{Stealth[length=1.9mm]}, line width=0.6pt, draw=cSlow},
        sgarr/.style={-{Stealth[length=1.9mm]}, line width=0.6pt, dashed, draw=cMute},
        shareline/.style={line width=0.6pt, dashed, draw=cMute},
        fastdash/.style={-{Stealth[length=1.9mm]}, line width=0.6pt, dashed, draw=cFast!85},
        midtxt/.style={inner sep=2pt, fill=white, text=black},
        sharetxt/.style={inner sep=1.5pt, fill=white, text=cMute!50!black},
        fasttxt/.style={inner sep=1.5pt, fill=white, text=cFast!85!black},
        ctxlbl/.style={align=center, text=cSlow},
    ]
        \foreach \i in {0,...,4} {
            \node[tokhi] (xSc\i) at (-4.46+0.38*\i,0) {};
        }
        \node[fit=(xSc0)(xSc4), draw=none, inner sep=0pt] (xS) {};
        \node[ctxlbl, below=2pt of xS] (xSlbl)
            {Local context (recent tokens)};

        \foreach \i in {0,...,7} {
            \node[tokbox] (xTc\i) at (1.42+0.38*\i,0) {};
        }
        \foreach \i in {8,...,12} {
            \node[tokhi] (xTc\i) at (1.42+0.38*\i,0) {};
        }
        \node[fit=(xTc0)(xTc12), draw=none, inner sep=0pt] (xT) {};
        \node[ctxlbl, below=2pt of xT] (xTlbl)
            {Context of past histories (longer context)};

        \node[encbox] (eS) at (-3.7,1.65) {Student Model\\(short context window)};
        \node[encbox] (eT) at ( 3.7,1.65) {Teacher Model\\(long context window)};

        \draw[arr] (xS.north) -- (eS.south);
        \draw[arr] (xT.north) -- (eT.south);

        \draw[shareline] (eS.east) -- (eT.west)
            node[sharetxt, pos=0.5] {shared and frozen};

        \node[predbox] (pred) at (-3.7,3.85)
            {MLP Down-Projection\\\textcolor{cFast!85!black}{\textit{fast weight}} $\W_{\mathrm{down}}$};

        \node[lossbox] (loss) at (3.7,3.85) {
            $\mathcal{L}^{\mathrm{CD}}_t =
            \frac{1}{2}\big\|\sg(y_t^{(T)}) -
            f_{\W_{\mathrm{down}}}\!\left(h_t^{(S)}\right)\big\|_2^2$};

        \draw[arr] (eS.north) --
            node[midtxt, left=1pt] {$h_t^{(S)}$}
            (pred.south);

        \draw[sgarr] (eT.north) --
            node[midtxt, right=1pt] {$\sg(y_t^{(T)})$}
            (loss.south);

        \draw[fastdash] ([yshift=2.2pt]pred.east) --
            ([yshift=2.2pt]loss.west);

        \draw[fastdash] ([yshift=-2.2pt]loss.west) --
            node[fasttxt, below=1pt] {TTT update}
            ([yshift=-2.2pt]pred.east);

    \end{tikzpicture}
    \caption{
\textbf{Overview of \ttcd.} 
A shared and frozen backbone is run with two context lengths: a short-window student that observes only recent tokens and a long-window teacher that observes a longer history. 
\ttcd treats the discrepancy between the teacher representation and the student's projected hidden state as a context-distillation loss, and uses this loss to update the down-projection fast weight at test time. The predictive contribution of the self-teacher relative to a short-context student is distilled into fast weights, which persists after the evidence leaves the local window. The learned information can be retrieved by later related queries, benefiting future predictions.
}
    \label{fig:ttcd}
\end{figure}

We validate the effectiveness of \ttcd with extensive experiments on synthetic and long-context language modeling tasks across different model scales. When trained from scratch, \ttcd consistently outperforms sub-quadratic and TTT baselines including sliding-window attention, Gated DeltaNet, DeltaNet, and IP-TTT, achieving the best sliding-window perplexity over context lengths up to 32K tokens, and its gains increase with context length. At 760M parameters, \ttcd also more than doubles IP-TTT on the long context RULER benchmark with the gap widening as context length grows. Beyond from-scratch training, applying \ttcd as continual pre-training to SmolLM2-360M/1.7B and LLaMA-3.1-8B improves RULER performance at 32K and 64K, demonstrating drop-in extensibility for existing pre-trained transformer-based models. Our results show a further step forward to the continual learning architecture in language models.







\section{Preliminary and Related Work}

In this section, we introduce the notation, mechanisms and background used throughout the paper. We recall the Test-Time Training (TTT) framework as fast-weight memory, and describe context distillation as a supervision principle, which transfers information from a stronger teacher with more contextual information to a weaker student with a smaller window size.

\begin{figure}[t]
    \centering
    \begin{minipage}[t]{0.48\textwidth}
    \subcaption{Sliding Window PPL of 340M Model}
        \includegraphics[width=\linewidth]{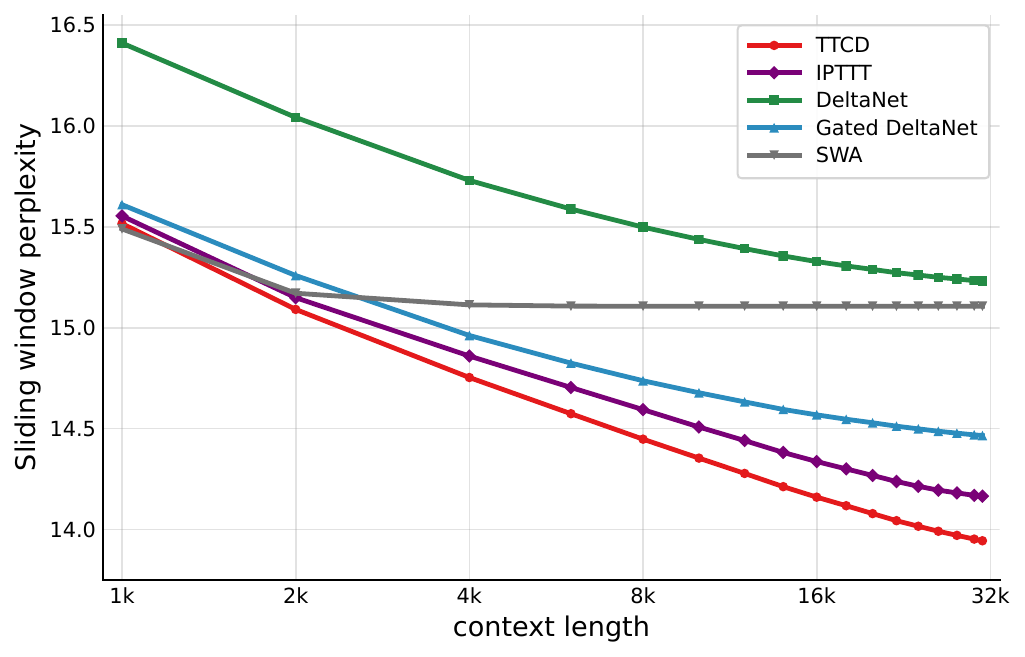}
    \end{minipage}
    \begin{minipage}[t]{0.48\textwidth}
        \subcaption{Sliding Window PPL of 760M Model}
        \includegraphics[width=\linewidth]{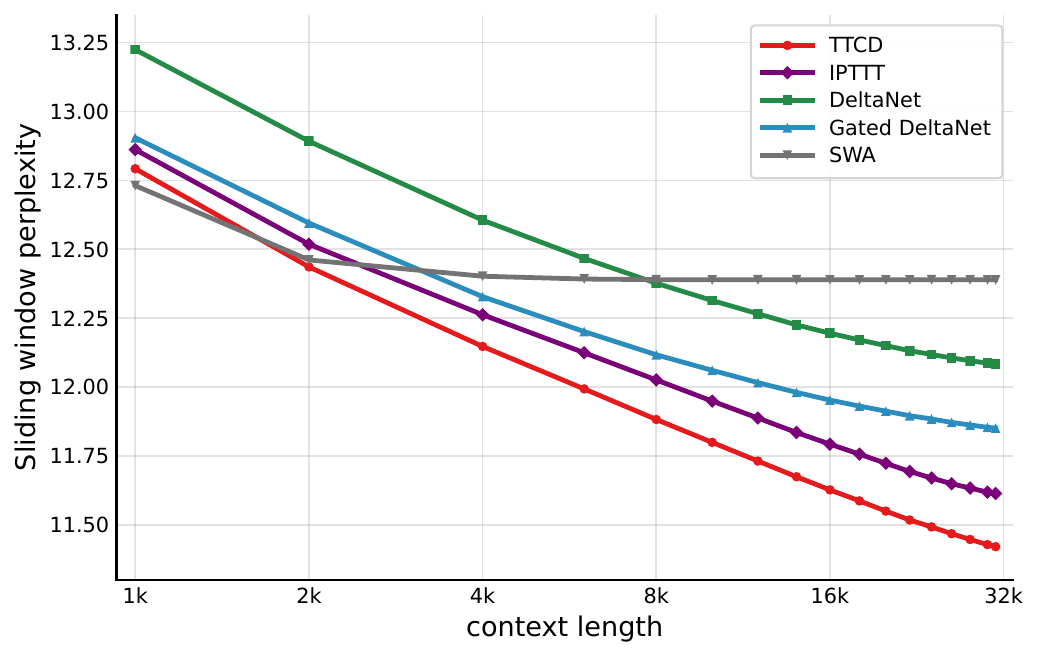}
    \end{minipage}
    \caption{Sliding window perplexity on the Books \citep{gao2020pile} at increasing context lengths. We show that \ipttcd consistently achieves lower perplexity than all the TTT-related baselines.}
    \label{fig:ppl_books}
\end{figure}

\subsection{Test-Time Training}

Recurrent neural networks have a long history in AI~\citep{hochreiter1997long, chung2014empirical, bai2018empirical}. Recent recurrent and test-time training architectures use fixed-size states or fast-weight updates as sub-quadratic alternatives to full attention~\citep{gu2021efficiently, fu2022hungry, peng2023rwkv, munkhdalai2024leave, gu2023mamba, dao2024transformers, yang2024parallelizing, yang2023gated}. DeltaNet and Gated DeltaNet use delta-rule key-value state updates~\citep{yang2024parallelizing,yang2023gated}; Mamba uses selective state-space recurrences~\citep{gu2023mamba}; and Log-Linear Attention and RATTENTION interpolate between linear and softmax attention to improve in-context recall~\citep{guo2025log,wang2025rattention}. While these models differ in their parameterization of the recurrent state, they share the view that long-context computation can be compressed into a compact dynamical memory rather than explicitly stored in the attention cache.

Test-time training (TTT) treats a set of sequence-specific \textit{fast weights} as the hidden state of a sequence model~\citep{ba2016using, schlag2021delta, sun2024learning, tandon2025end, feng2026inplace}. The fast weights are not frozen after training, but serve as dynamical states that encode the contextual information by consistently evolving under certain update rules.

Specifically, let $\theta$ denote the slow weights shared across sequences, and let $\W_t$ denote fast weights initialized at the beginning of each sequence and \(\mathcal{L}(\W,x)\) be the online loss function for input \(x\). Given an input sequence representation $(x_1,x_2,...,x_t)$ with current token $x_t$, the canonical TTT step consists of an \textbf{update} operation and a \textbf{apply} operation, where the update operation is computed by
\begin{equation}
    \label{eq:ttt-token-update}
    \W_{t} = \W_{t-1} - \eta \nabla_{\W} \mathcal L\bigl(\W_{t-1},x_t\bigr),
\end{equation}
thus $\W_t$ becomes a recurrent state over the time axis.  The \textbf{apply} operation \(o_t = f_{\W_t}(x_t)\) is done through a feature map \(f_{\W_t}(\cdot)\) associated with the fast weight \(\W_t\), which produces the final output. 

Under the TTT framework, many design choices are investigated, including different loss functions, optimizers and memory parameterizations \cite{yang2023gated, wang2025rattention, behrouz2024titans, sun2024learning}. Recently, there are two major competitive and efficient variants of TTT architectures\footnote{For computation efficiency, we only discuss the TTT variants without multi-layer backward propagation, e.g. TTT-E2E \cite{tandon2025end}, in this section.}: Large Chunk TTT (LaCT) \cite{zhang2025test} and In-Place TTT (IP-TTT) \cite{feng2026inplace}.

\paragraph{Large Chunk TTT.} For computation efficiency, practical TTT methods often replace per-token recurrence with mini-batch \cite{sun2024learning} or chunk-wise updates \cite{zhang2025test}. By partitioning the sequence into chunks of size $C$, and letting $\ell(\W,\cdot)$ denote the TTT loss at position $t$, the chunk update is
\begin{equation}
\label{eq:main}
\W_i = \W_{t-1} - \frac{\eta}{C}\sum_{t\in[(i-1)C, iC]}\nabla_{\W}\mathcal{L}(\W_{i-1}, x_t),
\end{equation}
for $i=1,\dots,T/C$ when sequence length $T$ is divisible by the chunk size $C$. In this form, the choice of $\ell(\cdot,\cdot)$ determines what information each chunk writes into the fast weights for later use.

\paragraph{In-Place TTT.} In-Place TTT adapts this framework to standard transformer language models by choosing existing model parameters as the fast weights~\citep{feng2026inplace}, shedding light on ``drop-in" designs for TTT without a fundamental architecture modification. In particular, In-Place TTT provided a bridge between \emph{fast} and \emph{slow} weights by reusing the native parameters (the down projection weights in the MLP layers), which is a fundamental step towards connecting temporary memory and long term memory. Our method follows this in-place formulation.

\subsection{Context Distillation and KV compression}
Context distillation \cite{snell2022learning} is a long-held method in language model post-training, aiming to internalize in-context knowledge into the model’s parameters. Specifically, context distillation tries to match the logit outputs or hidden states of a teacher model to those of a student model, where the teacher perceives more context than the student is allowed to see. Formally, let $h_t^{(T)}$ denote the teacher's latent representations and $h_t^{(S)}$ the student's. A standard hidden-state distillation objective is
\begin{equation}\label{eq-def:CD}
    \mathcal{L}_{t} = \big\|h_t^{(S)} -
    \sg\bigl(h_t^{(T)}\bigr) \big\|_2^2
\end{equation}
where $\textrm{stopgrad}(\cdot)$ detaches back-propagation of the teacher. Existing context-compression methods often use distillation to produce an offline prompt, memory, or compressed KV representation before generation~\citep{chuang2024learning, jiang2023llmlingua, yuan2025native, li2024snapkv, ge2023model, oren2024transformers, zhang2023h2o, lester2021power, li2021prefix}. Earlier work optimizes soft prompts for efficient adaptation~\citep{lester2021power, li2021prefix} or drops KV-cache entries to trade accuracy for efficiency~\citep{li2024snapkv, ge2023model, oren2024transformers, zhang2023h2o}. Newer long-context methods target million-token settings with more memory-efficient compression: Cartridges trains small document weights from synthetic QA pairs~\citep{eyuboglu2025cartridges} and performs offline, corpus-specific optimization to compress contexts; KV-Distill learns token dropping via full-cache distillation~\citep{chari2025kv}; Finch prunes caches using attention scores and calibration prompts~\citep{corallo2024finch}; and prior analyses find extractive compression competitive across tasks~\citep{jha2024characterizing}. These methods typically follow an offline two-stage pipeline, producing a fixed surrogate before generation and decoupling compression from the model’s current computation. In contrast, \ttcd is pretrained end-to-end to learn an online test-time update. These fast weights function as an evolving recurrent state, rather than as a fixed representation prepared before querying.

\section{Test-Time Context Distillation}
In this section, we introduce \textbf{T}est-\textbf{T}ime \textbf{C}ontext \textbf{D}istillation (\ttcd), which updates the model's fast weights using the context distillation objective. We first present our methodology on the context distillation training objective using shared weights teacher and student attentions (\Cref{subsec:method}), followed by a theoretical analysis on the benefit of our objective (\Cref{subsec: theory}). Finally, we introduce our practical implementation details for architectural design. 

\subsection{Methodology}
\label{subsec:method}

Prior TTT objectives often adopt a reconstruction objective to compress key-value pairs in the TTT fast weight \cite{sun2024learning, behrouz2024titans, zhang2025test} or an equivalent state space \citep{gu2023mamba,yang2023gated,yang2024gdn,yang2023gated}. Recent works \cite{feng2026inplace, tandon2025end} noticed the suboptimality of the reconstruction loss and proposed to use next-token prediction (NTP) loss. However, the two objectives face different drawbacks:
\begin{itemize}[leftmargin=2em]
    \item Reconstruction-type objectives often compress contexts uniformly, rather than allocating its capacity for more informative tokens. 
    \item NTP-type objectives often exhibit a recency bias, overemphasizing local signals near the prediction target.
\end{itemize}
Our idea is to combine context distillation with a forward-looking objective, encouraging the fast weights to retain the information most useful for \emph{future} predictions.

Following the prior works \cite{zhang2025test, feng2026inplace} which place parallel efficiency as top priority, we consider \emph{a chunk-wise formulation} of TTT as well. Let \(X=(x_1,\ldots,x_T)\) be processed in chunks \(\mathcal C_i\) of size \(C\). For each chunk, we run two causal sliding-window views of the same model. The teacher uses a long window of size \(w_T\), while the student uses a shorter window of size \(w_S<w_T\). Thus, at the same token position, the teacher can condition on earlier prefix tokens that are hidden from the student, and therefore contains predictive information that depends on remote history. Matching the student to the teacher turns this difference in the hidden states into a self-supervised memory target.

\paragraph{Hidden-state context distillation.} For a given layer and position \(t\), let \(h_t^{(T)}\) and \(h_t^{(S)}\) denote the pre-MLP hidden states produced by the teacher with window size \(w_T\), and student with window size $w_S$, and let \(y_t^{(T)} = \mathrm{MLP}(h_t^{(T)})\) and \(y_t^{(S)}=\mathrm{MLP}(h_t^{(S)})\) be the corresponding MLP outputs. We use the stop-gradient teacher output as the target for the student to match:
\begin{equation}\label{eq-def:chunk-CD}
    \mathcal{L}_t^{\mathrm{CD}} = \frac{1}{2}
    \|y_t^{(S)} - \sg\big(y_t^{(T)}\big)\|_2^2 ,\qquad 
    \mathcal{L}_{i}^{\mathrm{CD}} = \frac{1}{|\mathcal C_i|} \sum_{t\in\mathcal C_i}\mathcal{L}_t^{\mathrm{CD}}.
\end{equation}
where \(L_{i}^{\mathrm{CD}}\) is the chunk-wise aggregated  loss. Note that \ttcd performs \textbf{online} distillation that updates the model during test time, which is different from the traditional context distillation methods.


\paragraph{\ipttcd: Unifying slow and fast weights.} Throughout this paper, we focus on an instance of \ttcd: In-Place Test-Time Context Distillation (\ipttcd). It follows In-Place TTT \cite{feng2026inplace} that uses the MLP down-projection matrix as the fast weight for online adaptation. In a typical SwiGLU MLP layer, suppose the input hidden states are denoted by \(h_t^{(r)}\) at step $t$. The output \(y_t^{(r)}\) is computed by
\begin{align}
    y_t^{(r)} &\triangleq \mathrm{MLP} (h_t^{(r)}) \nonumber\\
    & \equiv \W_{\down} \Big( \W_{\mathrm{up}}h_t^{(r)} \odot \sigma(\W_{\mathrm{gate}}h_t^{(r)})\Big), \qquad r\in\{T,S\}. \label{eq-def:swiglu}
\end{align}
By considering the down projection weight $\W_{\down}$ as the online updated fast weight in chunk-wise TTT \cref{eq:main}, combined with hidden state context-distillation objective \cref{eq-def:chunk-CD}, we can obtain the update rule for \ipttcd. To be more specific, let's express the intermediate hidden states of MLP by
\begin{align*}
    z_t^{(r)}= \W_{\mathrm{up}}h_t^{(r)} \odot \sigma(\W_{\mathrm{gate}}h_t^{(r)}),\qquad y_t^{(r)}= \W_{\down}z_t^{(r)},
     \qquad t \in \mathcal{C}_i
\end{align*}
Let's write \(\bm{Z}_i^{(r)} = (z_t^{(r)})_{t \in \mathcal{C}_i}\) and \(\bm{Y}_i^{(r)} = (y_t^{(r)})_{t \in \mathcal{C}_i}\) to represent the chunks of \(y_t\) and \(z_t\). At the \(i\)-th chunk \(\mathcal C_i\), the current fast weight \(\W_{\down}^{(i-1)}\) is used to compute both teacher and student MLP outputs \(\bm{Y}_i^{(r)}, r \in \{T,S\}\). The negative gradient of the chunk loss with respect to the student fast weight is
\begin{align*}
    -\nabla_{\W_{\down}}\mathcal L_i^{\mathrm{CD}} 
    & = -\frac{1}{|\mathcal C_i|} \sum_{t\in\mathcal C_i}\nabla_{\W_{\down}}\frac{1}{2}\Big\|y_t^{(S)} -  \sg(y_t^{(T)}) \Big\|_2^2 \\
    & = -\frac{1}{|\mathcal C_i|}
    \Big[\bm{Y}_i^{(S)} - \sg(\bm{Y}_i^{(T)})\Big][\bZ_i^{(S)}]^\top
\end{align*}
The \ipttcd update for chunk \(\mathcal{C}_i\) is therefore
\begin{equation}
    \label{eq:inplace_update_sw}
    \W_{\down}^{(i)} = \W_{\down}^{(i-1)} - \frac{\eta_i}{|\mathcal C_i|}\Big[\bm{Y}_i^{(S)} - \sg(\bm{Y}_i^{(T)})\Big]
    [\bZ_i^{(S)}]^\top.
\end{equation}
The update direction is determined by the representation gap between long-window and short-window computation, while the write key is the student's own MLP activation. The teacher–student difference determines what information is written, while the student activation determines when similar information can be retrieved later. The updated down-projection matrix therefore acts as a sequence-specific memory rather than only as an adapted model parameter.


\subsection{Theoretical Analysis}
\label{subsec: theory}
After introducing the future utility intuition behind our method, the following analysis shows how in-place context distillation writes a long-context signal from an earlier query into the fast weights, which could be reusable at later related queries. For simplicity, we analyze a single final-layer MLP block. Following the chunked update in Equation~\eqref{eq:inplace_update_sw}, if position $n$ lies in chunk $j$, then all preceding chunks induce the fast-weight update
\begin{equation}
\Delta \W_{<j} = \eta \sum_{t \in \mathcal{I}_{<j}} \bigl(y_t^{(T)}-y_t^{(S)}\bigr)z_t^{(S)\top},
\label{eq:theory-update}
\end{equation}
where $\mathcal{I}_{<j}$ denotes the token positions in chunks strictly before chunk $j$. Constant factors from the squared distillation loss are absorbed into the learning rate $\eta$. Therefore, the student hidden state at position $n$ is shifted by
\begin{equation}
y_{n}^{(S)} \leftarrow y_{n}^{(S)}  + \Delta \W_{<j}z^{(S)}_n
\label{eq:theory-student-plus}
\end{equation}

Now consider a sequence containing an in-context learning episode,
\[
\X = [\cdots,\underbrace{x_{\contextstart},\ldots,x_{\contextend}}_{\text{context}},\cdots,x_{\querypos},\cdots,x_{\newquerypos}],
\]
where $(x_{\contextstart},\ldots,x_{\contextend})$ is the supporting context span, $\querypos$ is a representative answer-bearing position for the first query $q_1$, and $\newquerypos$ is the analogous position for a later related query $q_2$. The long-window teacher can use the context span to answer $q_1$, while the short-window student cannot.

\begin{assumption}\label{assump}
    We make the following two assumptions.
\begin{itemize}[leftmargin=2em]
    \item \textbf{Non-trivial in-context signal:} $\querypos-\lwindow < \contextstart \le \contextend \le \querypos-\swindow$ and $y_{\querypos}^{(T)\top}\bigl(y^{(T)}_{\querypos}-y^{(S)}_{\querypos}\bigr) \ge \hgap > 0$.
    \item \textbf{Chunk causality and related-query alignment:} if chunk $j$ contains $\newquerypos$, then $\querypos\in\mathcal{I}_{<j}$, $z^{(S)\top}_{\querypos}z^{(S)}_{\newquerypos} \ge \zalign > 0$, and $z_t^{(S)\top}z^{(S)}_{\newquerypos}=0$ for all $t \in \mathcal{I}_{<j}\setminus\{\querypos\}$.
\end{itemize}
\end{assumption}

The first assumption states that the teacher has a useful long-context signal at $\querypos$ that the student lacks. The second isolates a later query whose student activation reads the earlier signal from the fast weights. The following theorem guarantees the effective memorization of the distilled hidden state within the one-step update, making it possible for future utilization. Detailed proof see \Cref{appen_thm:main}.

\begin{proposition}[Context distillation transfers an in-context solution]
\label{thm:main}
    Under the setup and \Cref{assump}, for any learning rate $\eta>0$, the update induced by preceding chunks satisfies
\begin{equation}
\begin{aligned}
    {y}_{\querypos}^{(T)\top}\Delta {y}_{\newquerypos}^{(S)}
    \ge
    \eta\,\hgap\,\zalign.
\end{aligned}
\label{eq:theory-main}
\end{equation}
Equivalently, the chunk update moves the later student hidden state at $q_2$ toward the teacher hidden state that already encodes the in-context solution for $q_1$.
\end{proposition}

\begin{remark}
    \Cref{thm:main} shows that the distillation update stores the teacher-student residual from $q_1$ and replays it at an aligned later query $q_2$. Thus, fast weights can preserve a useful long-context computation after the supporting context leaves the student's attention window.
\end{remark}

\subsection{Practical Implementation}
\label{subsec:practical_implementation}
Based on the theoretical guarantees, we verify the feasibility of our method conceptually. To further improve the expressivity and efficiency of our method, we further consider certain practical implementations and add some lightweight components.

\paragraph{Off-policy update for context parallelism} Equation~\eqref{eq:inplace_update_sw} describes the ideal recurrent update, where each chunk is evaluated with the fast weight produced by all previous chunks. Directly materializing this recurrence limits parallelism, since the MLP outputs in chunk \(i\) depend on the fast weight produced by chunk \(i-1\). To improve parallelism over the context, our implementation uses an off-policy approximation: all teacher and student activations used to construct the update are computed with the base down-projection \(\W_{\down}^{(0)}\), while the resulting updates are accumulated causally.

\paragraph{Practical feature transformation} To further improve the performance of our model, we add some lightweight augmentation to the current TTCD update. We first adopt lightweight short causal convolutions \(\text{conv}_T\) and \(\text{conv}_S\) applied to teacher and student MLP activations, enhancing the model capability by introducing local mixing. For chunk \(\mathcal C_i\), let us recall that \(\bZ_i^{(r)}, r \in \{T,S\}\) are the intermediate activations of the SwiGLU MLP layers, we transform them by
\begin{equation}
    \bar{\bZ}_i^{(T)}=\text{conv}_T(\bZ_i^{(T)}),
    \qquad
    \bar{\bZ}_i^{(S)}=\text{conv}_S(\bZ_i^{(S)}).
\end{equation}
We then form a value residual and normalized features,
\begin{equation}
    \bm{R}_i
    =
    \W_{\down}^{(0)}\W_{\mathrm{proj}}
    (\bar{\bZ}_i^{(T)}-\bar{\bZ}_i^{(S)}),
    \qquad
    \bm{K}_i
    = \ell_2\text{-}\mathsf{norm}(\bar{\bZ}_i^{(S)}),
\end{equation}
where \(\W_{\mathrm{proj}}\) is a learned projection and \(\ell_2\text{-}\mathsf{norm}(\cdot)\) normalizes the student hidden states. The practical fast-weight correction is
\begin{equation}
    \Delta \W_{\down}^{(i)} = \beta_i \bm{R}_i \bm{K}_i^\top,
    \label{eq:fast-weight-update}
\end{equation}
followed by the causal parallel scan
\begin{equation}
    \W_{\down}^{(i)} = \W_{\down}^{(0)} + \sum_{1\leq j\leq i}\Delta \W_{\down}^{(j)}.
\end{equation}
This preserves the update structure of \Cref{eq:inplace_update_sw}: the residual \(\bm{R}_i\) specifies the discrepancy between the teacher and the student, while \(\bm{K}_i\) specifies where it should be written in the student activation.
\begin{table}[t]
\centering
\setlength{\tabcolsep}{3.2pt}
\resizebox{\linewidth}{!}{
\begin{tabular}{lcccc|ccccc|cccc}
\toprule
& \multicolumn{4}{c}{\textbf{NIAH (Avg)}} 
& \multicolumn{5}{c}{\textbf{Common Sense Reasoning}} 
& \multicolumn{4}{c}{\textbf{BABILong}} \\
\cmidrule(lr){2-5} \cmidrule(lr){6-10} \cmidrule(lr){11-14}
\textbf{Model} 
& \textbf{4K} & \textbf{8K} & \textbf{16K} & \textbf{32K}
& \textbf{HellaSwag} & \textbf{ARC-E} & \textbf{ARC-C} & \textbf{MMLU} & \textbf{Average}
& \textbf{4K} & \textbf{8K} & \textbf{16K} & \textbf{32K} \\
\midrule
DeltaNet 
& 41.67 & 38.25 & 24.83 & 19.33
& 45.81 & 48.61 & 28.67 & 26.44 & 37.38
& 16.99 & 8.85 & 3.90 & 1.94 \\
Gated DeltaNet 
& 32.21 & 29.88 & 22.29 & 18.62
& \textbf{48.15} & 54.76 & 29.44 & 25.25 & 39.40
& 19.24 & 11.93 & 6.08 & 2.84 \\
SWA 
& 44.88 & 22.58 & 10.79 & 4.21
& 47.32 & 52.53 & \textbf{30.72} & 25.74 & 39.08
& \textbf{25.42} & 15.45 & 7.91 & 4.40 \\
IP-TTT
& 59.83 & 33.88 & 23.92 & 9.29
& 47.26 & 53.45 & 29.35 & \textbf{26.89} & 39.24
& 21.26 & 14.87 & 7.59 & 4.46 \\
\ipttcd
& \textbf{65.79} & \textbf{46.92} & \textbf{41.75} & \textbf{21.96}
& 46.83 & \textbf{55.77} & 30.20 & 26.78 & \textbf{39.90}
& 23.32 & \textbf{16.69} & \textbf{10.37} & \textbf{6.73} \\
\bottomrule
\end{tabular}
}
\vspace{2pt}
\caption{Common-sense and long-context evaluation for 760M models. NIAH reports average RULER NIAH accuracy; Common Sense Avg is computed over HellaSwag, ARC-Easy, ARC-Challenge, and MMLU; BABILong reports average accuracy over five QA tasks.}
\label{tab:commonsense-niah-babilong-760m}
\end{table}

\begin{figure}[t]
\centering
\vspace{-0.5cm}
\begin{minipage}[t]{0.7\textwidth}
\vspace{0pt}
\centering
\subcaption{MAD Benchmark Results.}
\label{fig:mad}
\resizebox{\textwidth}{!}{%
{
\begin{tabular}{l|cccccc|c}
\toprule
 & \makecell{In-context\\Recall} & \makecell{Noisy\\Recall} & \makecell{Fuzzy\\Recall} & Compress & Memorize & \makecell{Selective\\Copy} & Overall \\
\midrule

DeltaNet & 99.9 & 86.7 & 25.1 & 42.3 & 55.5 & \textbf{100.0} & 67.8 \\
GDN      & 99.9 & 99.7 & 18.2 & 47.3 & 86.9 & 99.9 & 74.4 \\
LaCT     & \textbf{100.0} & \textbf{100.0} & 18.6 & 48.3 & 80.7 & 99.8 & 73.8 \\
IP-TTT    & 98.1 & 97.1 & 52.4 & 43.4 & 82.9 & 96.6 & 78.4 \\
\rowcolor{lightyellow}
\textbf{\ipttcd
}   & 99.3 & 98.8 & \textbf{56.2} & \textbf{50.0} & \textbf{88.9} & 99.2 & \textbf{82.1} \\
\bottomrule
\end{tabular}
}}
\end{minipage}
\hfill
\begin{minipage}[t]{0.28\textwidth}
\vspace{0pt}
\subcaption{760M S-NIAH-2}
\centering
\includegraphics[width=\textwidth]{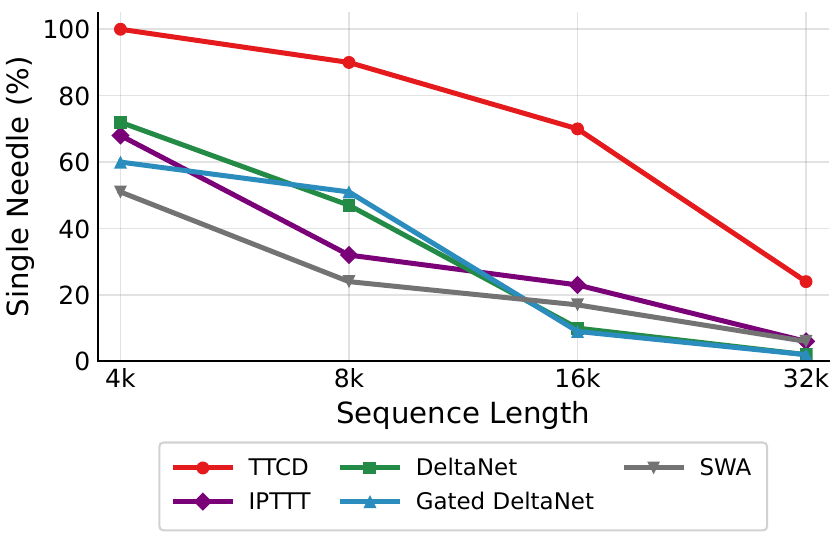}
\end{minipage}

\caption{Table \textbf{(a)} Results on the MAD benchmark~\citep{poli2024mechanistic}. \ipttcd attains the best overall score (82.1) and leads on the most memory-sensitive tasks (Fuzzy Recall, Compress, Memorize) that have not been saturated by prior baselines. Figure \textbf{(b)} Single-Needle-in-a-Haystack (S-NIAH-2) accuracy of 760M models from 4K to 32K. \ipttcd consistently outperforms all baselines and degrades far more gracefully as the context grows.}
\end{figure}

\section{Experiments}
In this section, we conduct extensive experiments to compare our method with various competitive TTT-related architectures and show the advantage of \ipttcd. Specifically, we aim to answer:

\begin{itemize}[leftmargin=0.5cm]
    \item Can IP-TTCD selectively preserve future-predictive past information under limited capacity?
    \item Can IP-TTCD outperform prior TTT approaches when trained from scratch?
    \item Can continual pre-training install the adaptability of \ttcd and enhance pre-trained models?
\end{itemize}
We first train \ipttcd on synthetic tasks (\Cref{subsec:synthetic}) requiring selective compression as a proof of concept. We then compare our method with the baselines on long-context language modeling pre-training tasks (\Cref{subsec:pt}). We also explore the possibility of enabling pre-trained transformer models to acquire the test-time adaptability capability through continual pre-training (\Cref{subsec:ct}). We finally examine the effects of the architecture components. 
\subsection{Synthetic Tasks}
\label{subsec:synthetic}
We first evaluate some synthetic tasks as a proof of concept to show {whether our model can selectively preserve useful information under limited capacity}. Specifically, we compare \ipttcd with other architectures on the MAD benchmark, a suite of synthetic sequence modeling tasks designed to probe whether an architecture can recall, denoise, compress, memorize, and selectively copy information under limited-context or recurrent-state constraints. The results are shown in Table~\eqref{fig:mad}. The MAD benchmark evaluates synthetic long-context abilities including recall, compression, memorization, and selective copying. On MAD, \ipttcd achieves the best overall score and outperforms all baselines, showing clear gains on the most memory-sensitive or fuzzy recall tasks.

\subsection{Pre-training From Scratch}
\label{subsec:pt}
We further evaluate the approach's long context performance when training a model from scratch. We analyze the long-context language modeling capabilities at two scales: 340M and 760M.

\paragraph{Setup} We use Long-Data-Collections \citep{long_data_collection} for pre-training with 32K context length. We compare \ipttcd against various baselines including (1) In-Place TTT with the LM-aligned Conv1D target constructed from input embeddings\cite{feng2026inplace}.\footnote{This setting uses IP-TTT’s original recommended implementation and the LM-aligned formulation. We observe instability in IP-TTT's training when the target is set as the hidden state of the current layer. Using input embedding as a target achieves clearly better performance. } (2) standard Transformer with sliding window attention (SWA) (3) DeltaNet and Gated DeltaNet \citep{schlag2021delta, yang2023gated}. All models are trained on sequences with a 32k sequence length. Please refer to \Cref{app:train-and-eval} for implementation details.

\paragraph{Evaluation} We first evaluate the long-context language modeling performance using \emph{sliding window perplexity} \citep{feng2026inplace} on the validation set, which measures perplexity on a fixed final block of tokens when extending the preceding context. For the 760M model, we further evaluate general downstream capabilities on four commonsense and knowledge benchmarks, including HellaSwag~\citep{zellers2019hellaswag}, ARC-Easy and ARC-Challenge~\citep{clark2018think}, and MMLU~\citep{hendrycks2020measuring}. We also evaluate long-context performance on RULER~\citep{hsieh2024ruler} and BABILong~\citep{kuratov2024babilong}. For RULER, we evaluate on six retrieval tasks across different sequence lengths, including Single and Multi-Key/Query/Value Needle in a Haystack tasks.

\paragraph{Results}
\Cref{fig:ppl_books} shows the sliding-window perplexity of models at different context lengths. 
Across both model scales, \ipttcd consistently achieves the lowest perplexity among all compared architectures, outperforming sliding-window attention, DeltaNet, Gated DeltaNet, and In-Place TTT across nearly the entire context range. 
The advantage becomes especially clear as the context length increases: while several baselines saturate or degrade at longer contexts, \ipttcd continues to benefit from additional context and reaches its best performance at the full 32K length. 

Beyond perplexity, \ipttcd also demonstrates stronger long-context retrieval while maintaining competitive general capability. 
As shown in \Cref{tab:commonsense-niah-babilong-760m}, \ipttcd achieves the best average commonsense score among all compared architectures, suggesting that the proposed adaptation does not compromise short-context reasoning. 
The advantage is especially clear on RULER NIAH, where \ipttcd obtains the highest average accuracy at every context length and more than doubles IPTTT at 32K, improving from 9.29 to 21.96. 
It also consistently outperforms the baselines on BABILong across all evaluated lengths beyond $4k$, confirming that the context-distillation objective better supports retention and reuse of distant information.



\begin{table}[t]
\centering
\small
\setlength{\tabcolsep}{3.8pt}
\begin{tabular}{lrrrrrr@{\hspace{1.2em}}rrrrrr}
\toprule
 & \multicolumn{6}{c}{SmolLM2-360M} & \multicolumn{6}{c}{SmolLM2-1.7B} \\
\cmidrule(lr){2-7}\cmidrule(lr){8-13}
Model
& 4K & 8K & 16K & 32K & 64K & Avg
& 4K & 8K & 16K & 32K & 64K & Avg \\
\midrule
Base
& \textbf{58.13} & 39.75 &  0.00 &  0.00 &  0.00 & 19.58
& \textbf{68.73} & 50.64 &  0.03 &  0.00 &  0.00 & 23.88 \\
Base (Yarn-8x)
& 35.59 &  7.21 &  6.28 &  6.18 &  4.23 & 11.90
& 43.54 & 14.74 &  9.69 &  8.61 &  7.07 & 16.73 \\
\midrule
CPT
& 50.53 & \textbf{42.74} & 39.92 & 27.24 & {10.09} & 34.10
& 63.74 & 57.35 & 54.74 & \textbf{44.18} & 13.38 & 46.68 \\
IP-TTT
& 50.20 & 42.43 & 40.47 & 29.32 & \textbf{10.35} & 34.55
& 67.41 & \textbf{60.13} & {54.89} & 41.53 & 14.44 & 47.68 \\
\rowcolor{lightyellow}
\textbf{IP-TTCD}
& 49.79 & 42.13 & \textbf{40.59} & \textbf{33.55} &  9.79 & \textbf{35.17}
& 67.61 & {59.58} & \textbf{55.50} & 42.96 & \textbf{19.17} & \textbf{48.96} \\
\bottomrule
\end{tabular}

\vspace{0.2cm}
\begin{tabular}{llrrrrrr}
\toprule
Base Model & Method & 4K & 8K & 16K & 32K & 64K & Avg \\
\midrule
\multirow{4}{*}{Llama-3.1-8B}
& Base
& \textbf{94.26} & 90.59 & 81.71 & 19.74 & 0.03 & 57.26 \\

& CPT
& 93.87 & \textbf{92.58} & 87.32 & 65.19 & 25.30 & 72.85 \\

& IP-TTT
& 93.95 & 92.31 & 88.04 & 66.17 & 26.50 & 73.39 \\

\rowcolor{lightyellow}
& \textbf{IP-TTCD}
& 93.72 & 92.32 & \textbf{88.59} & \textbf{66.48} & \textbf{27.65} & \textbf{73.75} \\
\bottomrule
\end{tabular}
\vspace{0.1cm}
\caption{Evaluation results on RULER benchmark \citep{hsieh2024ruler} based on continual pre-trained SmolLM2-360M/1.7B (Full attention), and LLaMA-3.1-8B. The scores are the average accuracy of all the 13 RULER tasks.}
\label{tab:smollm-ruler-clean-codebase-full}
\end{table}
\subsection{Continual Pre-training}
\label{subsec:ct}
Since training LLMs from scratch is highly
resource-intensive, we further show that \ttcd can be applied as a lightweight augmentation to pre-trained transformer-based LLMs. We start with open-sourced base models without intensive context extension training to verify the performance of our method, showing improved performance gain by continual training on long-context datasets.

\textbf{Setup.} We consider SmolLM2-360M/1.7B~\citep{allal2025smollm2} as the main base models, and additionally evaluate on Qwen3-0.6B/1.7B in Appendix~\ref{app:continual-pretraining-details}. For data, we use the 64K-context book subset from ProLong \citep{gao2025train} for long-context extension training. For more practical verification, we try our method on a larger scale model LLaMA-3.1-8B following the same continual learning protocol. We evaluate the long-context performance of the models on the RULER benchmark~\citep{hsieh2024ruler} with context length ranging from 4k to 64k. For baselines, we compare the \ipttcd enhanced model against two baselines (1) full attention transformers (with YaRN length extrapolation) and (2) IP-TTT enhanced full attention transformers.
More details of the settings are included in Appendix~\ref{app:train-and-eval}.

\textbf{Results.}  The results, summarized in \Cref{tab:smollm-ruler-clean-codebase-full} , show that the proposed adaptation methods consistently improve the long-context performance of SmolLM
models over the attention-only tuning baseline. \ipttcd achieves the
best overall average long-context performance and substantially improves performance at 32K in our 360M model and 64K in our 1.7B model.
These results indicate that \ipttcd is particularly effective in extending the usable
context range of the model. Additional experimental results and details can be found in Appendix~\ref{app:train-and-eval} and Appendix.~\ref{app:continual-pretraining-details}.

\subsection{Ablation Studies}
\label{sec:ablation}
We finally conduct the ablation studies on the design choices or hyperparameters of our method. Unless otherwise noted, we conduct the experiments on our 340M pre-training setting on Long Data Collections \citep{long_data_collection}.

\begin{figure}[t!]
\vspace{0.3cm}
    \centering
    \includegraphics[width=\linewidth]{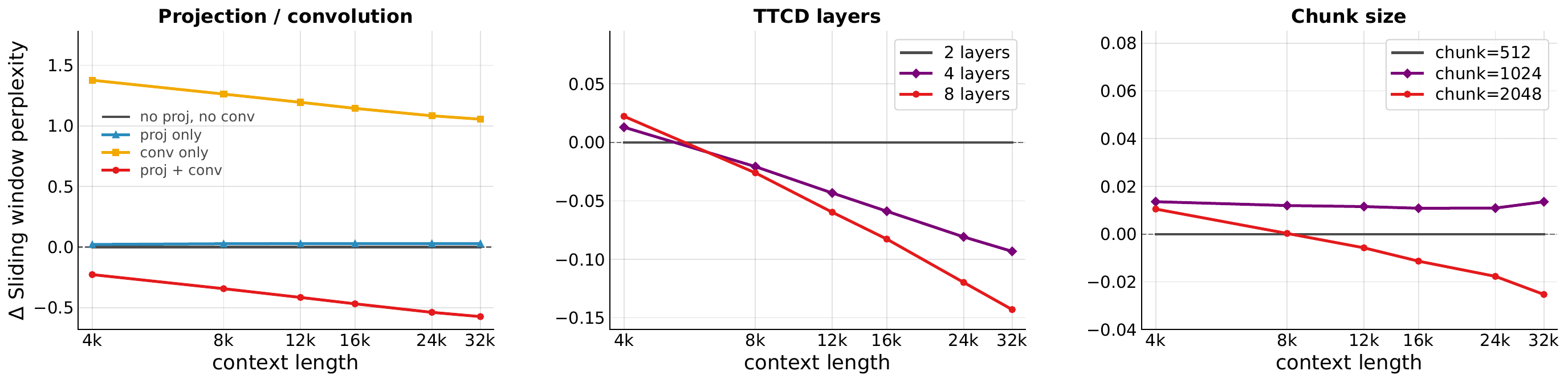}
    \caption{\textbf{Ablation studies on \ipttcd design choices.}
    We evaluate three key factors using 340M pre-training experiments on Long Data Collections.
    \textbf{Left:} ablating the causal convolutional feature maps and learned projection in the fast-weight update shows that combining both components yields the best long-context perplexity.
    \textbf{Middle:} increasing the number of \ipttcd layers consistently improves sliding-window perplexity, indicating that larger fast-weight capacity better supports long-context modeling.
    \textbf{Right:} increasing the chunk size from 512 to 2048 does not hurt performance and even slightly improves it, suggesting that \ipttcd remains effective under large-chunk updates.}
    \label{fig:ablation}
\end{figure}

\begin{figure}[t!]
    \centering
    \includegraphics[width=\linewidth]{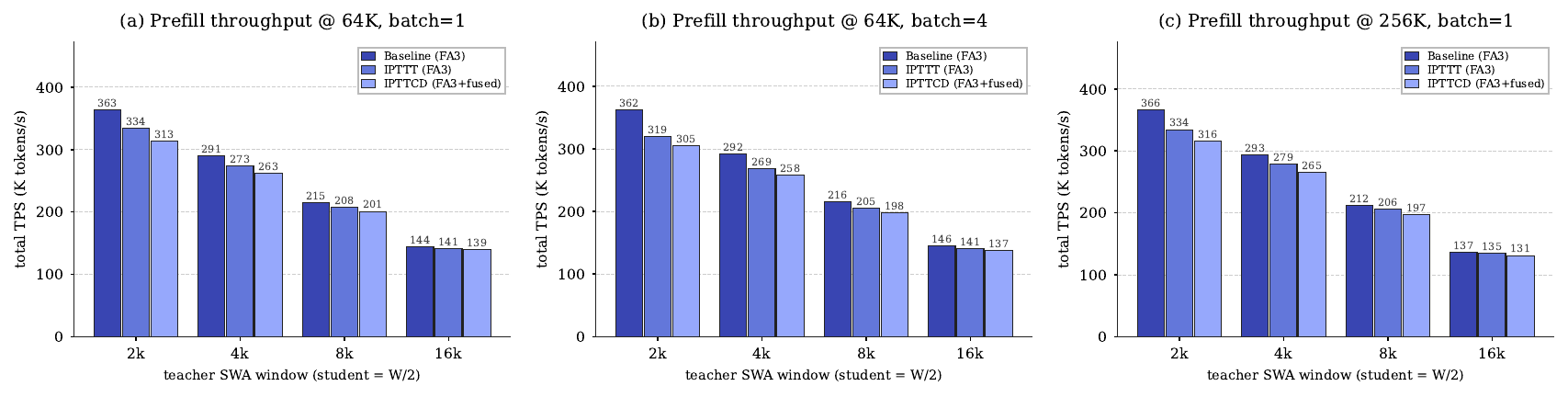}
    \caption{\textbf{Prefill throughput with our fused inference kernels.}
    FlashAttention-3 stack (\ipttcd uses the LSE-combination kernel); mean of four repeats on H100.
    Across all three settings, \ipttcd's overhead relative to IP-TTT is $1.3\%$--$6.6\%$ (\Cref{app:efficiency}).}
    \label{fig:prefill-fa3}
\end{figure}

\paragraph{Convolution and projection} We first investigate the effectiveness of the additional short convolution on the sequence \(\bar{\bZ}_i^{(T)}=\text{conv}_T(\bZ_i^{(T)})\), and the learnable linear projection \(\W_{\mathrm{proj}}\) on the test-time gradient update. The ablation experiments in \Cref{fig:ablation} (Left) show that the performance is the best when both parameters are used. The model with only convolution added in the update has a similar downward tendency in long-context perplexity as the model with both the convolution and the projection, but has much higher perplexity. Linear projection alone does not improve performance much, but it offers better training gain for convolution update. We conjecture that the phenomenon is related to the training dynamics but did not investigate further.

\paragraph{Number of \ipttcd layers} We then consider how the performance scales with the number of \ipttcd layers, which accounts for the state size of the implicit fast weight. \Cref{fig:ablation} (Mid) shows that the sliding window perplexity decreases clearly when the number of layers increases, and the steeper reduction also indicates a better utilization of models with more \ipttcd layers. 

\paragraph{Chunk size} We finally ablate on the chunk size of the update, which controls the efficiency and parallelism level of our method. Surprisingly, we found that within $\{512, 1024, 2048\}$, larger chunk size does not hurt the performance. Instead, the model with chunk size 2048 even has the best performance. This further confirms the scalability of large chunk update mentioned in \citep{zhang2025test, feng2026inplace}.  

\paragraph{Inference efficiency} At inference, \ipttcd computes two sliding-window attentions per \ipttcd layer (teacher and student), which we fuse into dedicated dual-window kernels; fusion speeds up end-to-end prefill by $10\%$--$14\%$ over the unfused implementation. As \Cref{fig:prefill-fa3} shows, the resulting prefill-time gap to IP-TTT is only $+1.3\%$--$+6.6\%$ on the FlashAttention-3 stack~\citep{shah2024flashattention3} ($-1.3\%$--$+3.2\%$ with FlashAttention-2~\citep{dao2023flashattention2}); the remaining cost stems from the student MLP path rather than attention. Kernel designs, full measurements, and numerical-equivalence validation are given in \Cref{app:efficiency}.


\section{Conclusion}
We introduced \ttcd, a test-time training framework that unifies past-context compression and online adaptation through a single context-distillation objective. Rather than asking the fast weights to store all the past information, \ttcd uses the hidden-state discrepancy between a long-window teacher and a short-window student as a dense self-supervised signal, encouraging the limited parametric state to store distant information that is most useful for future predictions. Our in-place instantiation, \ipttcd, realizes this idea as a low-rank correction to the existing MLP down-projection. We validate our method in both pre-training from scratch and continual pre-training, and find that \ipttcd consistently improves long-context 
language modeling, retrieval, and reasoning over strong baselines. By repeatedly distilling what the short-window model misses into fast weights, \ipttcd allows a model to acquire and reuse information as a sequence unfolds. We therefore view TTCD as a step toward architectural continual learning: the model not only adapts at test time, but also learns which parts of its experience should remain available for future predictions.

\paragraph{Limitations.} While \ipttcd improves long-context performance, several 
limitations remain. A naive two-kernel implementation of the teacher--student attention slows down prefill by up to $1.3\times$ at small teacher windows. With the fused inference kernels of \Cref{app:efficiency}, the 64K prefill overhead relative to IP-TTT drops to $-1.3\%$--$+3.2\%$ on the FlashAttention-2 stack and $+1.3\%$--$+6.6\%$ on FlashAttention-3; training remains $1.1\times$--$1.2\times$ slower than standard Transformer baselines. Our from-scratch pre-training experiments are limited to 760M parameters, although our continual-pretraining experiments extend to 8B models. Scaling \ipttcd from scratch to substantially larger models remains open.

\bibliography{neurips_2026}
\newpage
\appendix

\section{Experiment Details}
\subsection{Details of Training And Evaluation}
\label{app:train-and-eval}
\paragraph{Datasets} For the pre-training experiments, we use the open-sourced Long Data Collection \citep{long_data_collection}. Specifically, since the original dataset on huggingface is down, we used a public open-sourced copy~\footnote{\href{https://huggingface.co/datasets/emozilla/Long-Data-Collections-Pretrain-Without-Books}{https://huggingface.co/datasets/emozilla/Long-Data-Collections-Pretrain-Without-Books}}. To make up the missing part and keep the data mixing ratio similar, we find the open-source Books\footnote{https://huggingface.co/datasets/Geralt-Targaryen/books3/tree/main} dataset and mix it with the long-data-collection dataset with ratio 2:8. For evaluation, we also used another leave-out subset of the open-sourced Books dataset.
For continual pre-training experiments, we used the tokenized 64K book subset from ProLong \citep{gao2025train}.

\textbf{Training Details} All models are trained on NVIDIA H100 GPUs, detailed hyperparameters are listed below. For pre-training we apply warmup-stable-decay (WSD) learning rate schedule \citep{hu2024minicpm} which couples better with Muon \citep{wen2025fantastic}. For pre-training, the teacher convolution is zero initialized, and student convolution is random initialized. For continual pre-training, we use a cosine learning rate decay to 0.1 times the peak learning rate after warm-up steps.
\begin{table}[h]
\centering
\label{tab:pretraining-hyperparameters}
\begin{tabular}{lccc}
\toprule
\textbf{Hyperparameter} & \textbf{340M Model} & \textbf{760M Model} & \textbf{Continual Training}\\
\midrule
Optimizer & Muon & Muon & AdamW \\
Learning Rate & $3 \times 10^{-3}$ & $3 \times 10^{-3}$ &$1\times 10^{-5}$\\
Batch Size & 500K tokens & 1M tokens & 1M tokens\\
Weight Decay & 0.1 & 0.1 & 0.1\\
Gradient Clipping & 1.0 & 1.0 & 1.0\\
Warmup Steps & 500 & 500 & 200\\
Sequence Length & 32768 & 32768 & 65536\\
Tokens Trained & 20B & 40B & 10B\\
Sliding Window Size & 2048 & 2048 & Full (8192 for SWA)\\
Student Window Size & 1024 & 1024 & 2048\\
Hidden State Size & 1024 & 1536 & See base model\\
Number of TTT layers & 4& 4& 4\\
TTT learning rate & 0.3 & 0.3 & 0.3\\
\bottomrule
\end{tabular}
\vspace{0.1cm}
\caption{Pre-training hyperparameters for 340M and 760M models, and continual pre-training. *For continual pre-training, we remove the normalization on $\bZ_i^{(S)}$ for stability.}
\end{table}

\paragraph{Evaluation Details} We use the official lm-evaluation-harness \citep{biderman2024lessons} for all commonsense reasoning, BABILong, and RULER benchmarks. 
For commonsense reasoning, we evaluate on HellaSwag, ARC-Easy, ARC-Challenge, and MMLU, and report accuracy following the default evaluation protocol of lm-evaluation-harness. 
For RULER, we evaluate long-context retrieval performance at different context lengths and report average accuracy over the corresponding Needle-in-a-Haystack tasks. 
For BABILong, we report the average accuracy over the evaluated QA tasks at each sequence length. 
For sliding-window perplexity, we follow the evaluation protocol of \citet{feng2026inplace}: given a fixed final block of tokens, we progressively increase the amount of preceding context and compute the perplexity only on the final block. 
For the MAD benchmark, we follow the standard MAD evaluation protocol and report the best test accuracy after sweeping learning rates and weight decay for each task configuration.

\subsubsection{Synthetic MAD Benchmark Setup} 
\label{app:mad-setup}

\paragraph{Model architecture.}
All models follow the same two-layer architecture to ensure a fair comparison. Each layer consists of a sequence mixing sublayer followed by a channel mixing sublayer (SwiGLU MLP), wrapped in a Pre-LN residual block with RMSNorm. The hidden dimension is $d=128$ throughout. The SwiGLU MLP uses an intermediate dimension of $\frac{8d}{3}$ rounded to the nearest multiple of 256, resulting in $d_{\mathrm{ff}}=256$. All models use 4 attention heads with head dimension 32.

\paragraph{Sliding window configuration.}
Since several MAD tasks use short sequences (e.g., Compression and Memorization use $L \leq 256$ and $L = 32$, respectively), a single window size does not yield a meaningful comparison across all tasks. For IP-TTCD, IP-TTT and LaCT, we use task-dependent window sizes:
\begin{itemize}
    \item \textbf{ICR, Noisy-ICR, Fuzzy-ICR, Selective-Copying}: Teacher window size $W=32$, student/TTT chunk size $C=16$, student visible window $= 1 \times C = 16$.
    \item \textbf{Compression, Memorization}: Teacher window size $W=16$, student/TTT chunk size $C=8$, student visible window $= 1 \times C = 8$.
\end{itemize}
GDN and DeltaNet are recurrent models and do not use a sliding window.

\paragraph{Training.}
All models are trained with AdamW and a cosine learning rate schedule (minimum LR $10^{-6}$) for 200 epochs with batch size 128 in BF16 mixed precision. We sweep over learning rates $\{10^{-4}, 5 \times 10^{-4}, 10^{-3}\}$ and weight decays $\{0, 0.1\}$, yielding 6 hyperparameter configurations per task setting. We report the best test accuracy across the hyperparameter sweep for each (task, vocab\_size, seq\_len) configuration, following the standard MAD protocol. The total benchmark comprises 396 configurations across 6 tasks. All experiments use seed 12345.
\subsection{Additional Results}

\subsubsection{Pre-training}

\paragraph{Per-task RULER NIAH breakdown.}
\Cref{tab:ruler-niah-detail-760m} reports per-subtask accuracy for the 760M 
models from \Cref{subsec:pt}. \ipttcd attains the highest average at every 
context length, and the gap to the baselines widens with length: at 32K, 
\ipttcd reaches 21.96, while IP-TTT-Embedding drops to 9.29 and all 
recurrent baselines stay below 20. The advantage is most pronounced on 
the harder retrieval tasks (S2, S3, MK1, MV, MQ), which require 
disambiguating distractors or composing multiple keys or values. SWA 
degrades sharply beyond its 2K training window, as expected.
\begin{table}[h]
\centering
\small
\setlength{\tabcolsep}{5.2pt}
\resizebox{\linewidth}{!}{
\begin{tabular}{llrrrrrrr}
\toprule
\textbf{Model} & \textbf{Context}
& \textbf{S1} & \textbf{S2} & \textbf{S3}
& \textbf{MK1} & \textbf{MV} & \textbf{MQ} & \textbf{Avg} \\
\midrule
\multirow{4}{*}{IPTTCD}
& 4K  & 48.00 & \textbf{100.00} & \textbf{91.00} & \textbf{71.00} & \textbf{43.50} & \textbf{41.25} & \textbf{65.79} \\
& 8K  & 23.00 & \textbf{90.00}  & \textbf{75.00} & 41.00 & \textbf{27.00} & \textbf{25.50} & \textbf{46.92} \\
& 16K & 14.00 & \textbf{70.00}  & \textbf{74.00} & 42.00 & \textbf{25.00} & \textbf{25.50} & \textbf{41.75} \\
& 32K & 1.00  & \textbf{24.00}  & \textbf{29.00} & \textbf{33.00} & \textbf{21.50} & \textbf{23.25} & \textbf{21.96} \\
\midrule
\multirow{4}{*}{In-Place TTT}
& 4K  & 96.00 & 68.00 & 69.00 & 70.00 & 30.25 & 25.75 & 59.83 \\
& 8K  & 51.00 & 32.00 & 27.00 & \textbf{52.00} & 25.25 & 16.00 & 33.88 \\
& 16K & 23.00 & 23.00 & 10.00 & \textbf{45.00} & 24.25 & 18.25 & 23.92 \\
& 32K & 0.00  & 6.00  & 7.00  & 21.00 & 15.00 & 6.75  & 9.29 \\
\midrule
\multirow{4}{*}{DeltaNet}
& 4K  & \textbf{100.00} & 72.00 & 6.00  & 40.00 & 14.75 & 17.25 & 41.67 \\
& 8K  & \textbf{100.00} & 47.00 & 15.00 & 36.00 & 15.00 & 16.50 & 38.25 \\
& 16K & \textbf{100.00} & 10.00 & 3.00  & 21.00 & 6.75  & 8.25  & 24.83 \\
& 32K & \textbf{100.00} & 2.00  & 3.00  & 6.00  & 1.25  & 3.75  & 19.33 \\
\midrule
\multirow{4}{*}{Gated DeltaNet}
& 4K  & \textbf{100.00} & 60.00 & 3.00 & 17.00 & 6.00 & 7.25 & 32.21 \\
& 8K  & \textbf{100.00} & 51.00 & 2.00 & 12.00 & 5.00 & 9.25 & 29.88 \\
& 16K & \textbf{100.00} & 9.00  & 1.00 & 12.00 & 8.50 & 3.25 & 22.29 \\
& 32K & 98.00 & 2.00 & 2.00 & 4.00 & 2.75 & 3.00 & 18.62 \\
\midrule
\multirow{4}{*}{SWA}
& 4K  & 48.00 & 51.00 & 49.00 & 57.00 & 32.00 & 32.25 & 44.88 \\
& 8K  & 23.00 & 24.00 & 24.00 & 27.00 & 19.00 & 18.50 & 22.58 \\
& 16K & 11.00 & 17.00 & 6.00  & 10.00 & 10.50 & 10.25 & 10.79 \\
& 32K & 1.00  & 6.00  & 3.00  & 5.00  & 5.25  & 5.00  & 4.21 \\
\bottomrule
\end{tabular}
}
\vspace{0.2cm}
\caption{Detailed RULER synthetic NIAH results for 760M models. S1, S2, and S3 denote the three single-needle retrieval tasks; MK1 denotes multikey retrieval; MV denotes multivalue retrieval; MQ denotes multiquery retrieval. Scores are averaged over 100 examples per task and reported as percentages. Avg is the mean over the six NIAH subtasks at each context length. Higher is better.}
\label{tab:ruler-niah-detail-760m}
\end{table}

\subsubsection{Continual Pre-training}
\label{app:continual-pretraining-details}

\paragraph{Base models.}
We use SmolLM2-360M and SmolLM2-1.7B~\citep{allal2025smollm2} as base 
Transformer models. Both are originally trained with an 8K context window. 
We compare the original base model, YaRN-based length extrapolation, 
attention-only continual pre-training (CPT), IP-TTT~\citep{feng2026inplace}, 
and our IP-TTCD augmentation. We additionally evaluate continual pre-training 
on Qwen3-0.6B and Qwen3-1.7B, whose base models already support long context 
natively, as a complementary stress test of our method.

\paragraph{Training setup.}
For continual pre-training, we use the tokenized 64K book subset of 
ProLong~\citep{gao2025train}. All models are trained with sequence length 
64K for 10B tokens. We use AdamW with learning rate $1\times 10^{-5}$, batch 
size 1M tokens, weight decay 0.1, gradient clipping 1.0, and 200 warmup steps, 
followed by a cosine decay schedule to 0.1 times the peak learning rate. For 
IP-TTT and IP-TTCD, we insert 4 TTT layers and use a TTT learning rate of 0.3. The teacher shortconv is initialized as identity, while the student convolution is initialized to zero in order to keep the model behavior the same before training.

\paragraph{Evaluation setup.}
We evaluate long-context performance on RULER~\citep{hsieh2024ruler} with 
context lengths 4K, 8K, 16K, 32K, and 64K. For Qwen3, whose native context length is 32K, we report results up to 32K. We report the average accuracy 
over all 13 RULER tasks. All evaluations are conducted with the official 
lm-evaluation-harness implementation~\citep{biderman2024lessons}.

\paragraph{Additional results.}
In addition to the averaged RULER results reported in the main text, we 
provide the full breakdown for SmolLM2 in \Cref{tab:smollm-ruler-full-app}, 
covering both full attention and sliding window attention (SWA) inference 
settings. We further report results for Qwen3 base models in 
\Cref{tab:qwen-ruler-app}. Together, these results show that IP-TTCD extends 
the usable context range of base Transformer models, especially at longer 
context lengths where the original models and simple length-extrapolation 
baselines degrade sharply, and that the gains transfer across base models 
with very different long-context starting points.

\paragraph{IP-TTCD recovers long-context performance under full attention.}
Under full-attention inference, \ipttcd achieves the highest RULER average on 
both SmolLM2 scales, with the most visible gains at the longest contexts 
(e.g.\ 33.55 vs.\ 27.24 / 29.32 at 32K on 360M, and 19.17 vs.\ 13.38 / 14.44 
at 64K on 1.7B), suggesting that the long-window distillation signal is most 
useful exactly where short-window supervision becomes least informative. 
Under SWA inference, where the long-range attention path is removed at test 
time, the three continual pre-training variants (CPT, IP-TTT, \ipttcd) 
perform within roughly one RULER point of each other, which is consistent 
with the design of \ipttcd: the distillation signal is informative when the 
long-window path is available, and benign when it is not.

\paragraph{Robustness across different base models.}
We also evaluate \ipttcd on Qwen3-0.6B and Qwen3-1.7B, which already have 
strong long-context capability out of the box. The pre-training data of 
Qwen3 is known to be of high quality, and continual pre-training on a new 
long-context corpus slightly degrades Qwen3-1.7B's performance. Within this 
regime, \ipttcd retains the highest average among continual pre-training 
variants on both Qwen3-0.6B and Qwen3-1.7B, indicating that it behaves 
consistently across base models with different long-context starting 
points.

\subsubsection{Additional Experiments: Scaling up to 8B Models}

To extend our results to a larger scale, we start from LLaMA-3.1-8B and conduct the additional continual pre-training experiments.

\begin{table}[h]
    \centering
    \label{tab:llama-ct-architecture}
    \small
    \setlength{\tabcolsep}{5pt}
    \begin{tabular}{@{}p{0.38\linewidth}p{0.54\linewidth}@{}}
        \toprule
        Hyperparameter & Value \\
        \midrule
        Base model & LLaMA-3.1-8B \\
        CPT architecture & IP-TTT and IP-TTCD \\
        Transformer layers & 32 \\
        Hidden size & 4096 \\
        MLP intermediate size & 14336 \\
        Context length & 65,536 \\
        RoPE base $\theta$ & 500,000 \\
        TTT layers & $\{0, 8, 16, 24\}$ \\
        TTT chunk size & 2048 \\
        TTT learning rate & 0.3 \\
        TTT target & Hidden state \\
        Tokens per optimizer step & 2,097,152 \\
        Learning-rate schedule & Cosine decay \\
        Peak / final learning rate & $5\times10^{-6}$ / 0 \\
        Warmup & 200 optimizer steps \\
        Target training budget & 10,000,000,000 tokens \\
        \bottomrule
    \end{tabular}
    \vspace{0.3cm}
    \caption{Architecture and IP-TTT configuration of the 64K continual-pre-training run.}
\end{table}

\textbf{Training setup.}
We use 10B tokens in the tokenized 
ProLong~\citep{gao2025train} dataset with 64K context length, but \textbf{use all subsets} as a more practical settings. The training tokens are mixed based on the original datasets' (token-count) ratio. All models are trained with sequence length 
64K for 10B tokens. We use AdamW with learning rate $5\times 10^{-6}$, batch 
size 1M tokens, weight decay 0.1, gradient clipping 1.0, and 200 warmup steps, 
followed by a cosine decay schedule to 0.1 times the peak learning rate. For 
IP-TTT and IP-TTCD, we insert 4 TTT layers. The training details are included in the \Cref{tab:llama-ct-architecture} below.

\begin{table}[h]
\centering
\small
\setlength{\tabcolsep}{4pt}
\begin{tabular}{llrrrrrrr}
\toprule
\textbf{Model} & \textbf{Attn.} & \textbf{4K} & \textbf{8K} & \textbf{16K} & \textbf{32K} & \textbf{64K} & \textbf{Avg$_{\text{4-32K}}$} & \textbf{Avg$_{\text{all}}$} \\
\midrule
\multicolumn{9}{c}{\textit{SmolLM2-360M}} \\
\midrule
Base            & ---  & 58.13 & 39.75 &  0.00 &  0.00 &  0.00 & 24.47 & 19.58 \\
Base + YaRN-8x  & ---  & 35.59 &  7.21 &  6.28 &  6.18 &  4.23 & 13.82 & 11.90 \\
\midrule
CPT             & Full & \textbf{50.53} & \textbf{42.74} & 39.92 & 27.24 & 10.09 & 40.11 & 34.10 \\
IP-TTT          & Full & 50.20 & 42.43 & 40.47 & 29.32 & \textbf{10.35} & 40.61 & 34.55 \\
\rowcolor{lightyellow}
\textbf{IP-TTCD}& Full & 49.79 & 42.13 & \textbf{40.59} & \textbf{33.55} &  9.79 & \textbf{41.52} & \textbf{35.17} \\
\midrule
CPT             & SWA  & 56.33 & 38.36 & 18.23 & 12.42 & \textbf{8.16} & 31.34 & 26.70 \\
IP-TTT          & SWA  & \textbf{58.40} & \textbf{38.83} & 18.26 & 12.90 &  8.10 & \textbf{32.10} & \textbf{27.30} \\
\rowcolor{lightyellow}
\textbf{IP-TTCD}& SWA  & 57.41 & 38.35 & \textbf{19.55} & \textbf{13.00} &  7.21 & 32.08 & 27.10 \\
\midrule
\multicolumn{9}{c}{\textit{SmolLM2-1.7B}} \\
\midrule
Base            & ---  & 68.73 & 50.64 &  0.03 &  0.00 &  0.00 & 29.85 & 23.88 \\
Base + YaRN-8x  & ---  & 43.54 & 14.74 &  9.69 &  8.61 &  7.07 & 19.14 & 16.73 \\
\midrule
CPT             & Full & 63.74 & 57.35 & 54.74 & \textbf{44.18} & 13.38 & 55.00 & 46.68 \\
IP-TTT          & Full & 67.41 & 60.13 & {54.89} & 41.53 & 14.44 & 55.99 & 47.68 \\
\rowcolor{lightyellow}
\textbf{IP-TTCD}& Full & \textbf{67.61} & \textbf{59.58} & 55.50 & 42.96 & \textbf{19.17} & \textbf{56.41} & \textbf{48.96} \\
\midrule
CPT             & SWA  & 68.58 & 47.85 & \textbf{25.22} & \textbf{13.44} & \textbf{8.66} & 38.77 & \textbf{32.75} \\
IP-TTT          & SWA  & \textbf{69.84} & \textbf{48.39} & 25.00 & 13.18 &  7.36 & \textbf{39.10} & \textbf{32.75} \\
\rowcolor{lightyellow}
\textbf{IP-TTCD}& SWA  & 69.29 & 46.53 & 24.67 & 13.32 &  8.16 & 38.45 & 32.39 \\
\bottomrule
\end{tabular}
\vspace{0.2cm}
\caption{Complete RULER evaluation of continual pre-training on SmolLM2-360M 
and SmolLM2-1.7B, under both full attention and sliding window attention (SWA) 
inference settings. \textit{Base} and \textit{Base + YaRN-8x} are evaluated 
without continual pre-training and serve as references. 
Avg$_{\text{4-32K}}$ averages over 4K--32K, and Avg$_{\text{all}}$ averages 
over all five lengths. \textbf{Bold} marks the best result within each 
\{model size, attention setting\} block.}
\label{tab:smollm-ruler-full-app}
\end{table}
\begin{table}[h]
\centering
\small
\setlength{\tabcolsep}{6pt}
\begin{tabular}{llrrrrr}
\toprule
\textbf{Model} & \textbf{Method} & \textbf{4K} & \textbf{8K} & \textbf{16K} & \textbf{32K} & \textbf{Avg} \\
\midrule
\multirow{4}{*}{Qwen3-0.6B}
& Raw              & 80.98 & 72.23 & 65.91 & 59.69 & 69.70 \\
\cmidrule(lr){2-7}
& CPT              & 81.25 & 71.92 & 66.98 & 60.76 & 70.23 \\
& IP-TTT           & 80.44 & \textbf{72.69} & \textbf{68.03} & 59.60 & 70.19 \\
\rowcolor{lightyellow}
& \textbf{IP-TTCD} & \textbf{82.27} & 72.52 & 66.97 & \textbf{60.89} & \textbf{70.66} \\
\midrule
\multirow{4}{*}{Qwen3-1.7B}
& Raw              & 89.17 & 83.63 & 78.27 & 70.32 & 80.35 \\
\cmidrule(lr){2-7}
& CPT              & 87.21 & 81.47 & 74.87 & 67.69 & 77.81 \\
& IP-TTT           & \textbf{87.35} & 81.51 & \textbf{75.60} & 66.84 & 77.83 \\
\rowcolor{lightyellow}
& \textbf{IP-TTCD} & 87.20 & \textbf{82.65} & 74.56 & \textbf{67.94} & \textbf{78.09} \\
\bottomrule
\end{tabular}
\vspace{0.2cm}
\caption{RULER evaluation of continual pre-training on Qwen3-0.6B and 
Qwen3-1.7B. \textit{Raw} denotes the original base model without continual 
pre-training, which already supports long context natively. The remaining 
rows are continual pre-trained on the same long-context corpus as the 
SmolLM2 experiments. \textbf{Bold} marks the best result among continual 
pre-trained variants within each model size.}
\label{tab:qwen-ruler-app}
\end{table}

\section{Inference Efficiency: Fused Dual-Window Attention}
\label{app:efficiency}

This appendix details the inference kernels summarized in \Cref{sec:ablation}, together with complete measurements and the numerical-equivalence validation.

\paragraph{Why prefill costs extra.}
During prefill, every \ipttcd layer runs two sliding-window attentions over shared $Q,K,V$---the $\lwindow$-window teacher and the $\swindow$-window student ($\swindow=\lwindow/2$ in all our recipes)---followed by the chunk-wise fast-weight scan in the MLP. A naive implementation issues two FlashAttention calls per layer and recomputes the student's attention from scratch, even though the student's key range is a strict subset of the teacher's. The extra work (student attention, student MLP activations, short convolutions, $\Delta\W$ accumulation) is linear in sequence length and independent of $\lwindow$, while the baseline attention cost scales with $\lwindow$, so the \emph{relative} overhead grows as the teacher window shrinks: with off-the-shelf kernels we measure $+18\%$ prefill time versus IP-TTT at $\lwindow=2048$ and $+13\%$ at $\lwindow=8192$ under the $\swindow=\lwindow/2$ recipe (64K context, H100), and up to $+28\%$ with the larger student windows of earlier configurations---the origin of the ``up to $1.3\times$'' slowdown quoted in the limitations.

\paragraph{Fused dual-window attention.}
The student's window is contained in the teacher's, so every score the student needs is already computed during the teacher's pass. Our kernel, implemented in Triton, maintains two sets of online-softmax accumulators (output, running maximum, normalizer) per query block and sweeps the teacher's KV range once, backward from the causal diagonal. Within the shared range the student's $PV$ product is recovered from the teacher's by the row-wise rescaling identity $P_S = P_T\, e^{m_T-m_S}$, so keys and values are loaded once, $QK^\top$ is computed once, and the second output is produced with a handful of extra vector operations. The loop is split into five segments (causal diagonal, shared mask-free interior, student boundary, teacher-only interior, teacher boundary) so the bulk of the sweep runs without any masking; numerics follow the CUTLASS FlashAttention-2 implementation step by step (base-2 exponentials, the same guard for fully-masked rows, normalization by reciprocal multiplication), yielding outputs that are bit-identical to FlashAttention-2 on $99.7$--$99.9\%$ of elements and within 1--2 units in the last place (ulps) in bf16 elsewhere. We additionally replace the depthwise short convolutions of the TTT layers with a dedicated kernel that operates directly on the chunked layout; it is bitwise-equal to the cuDNN path and $11.6\times$ faster once the layout copies are included.

Interestingly, this kernel is faster than FlashAttention-2 \emph{even while doing strictly more work}: FlashAttention-2's local-attention path applies the sliding-window mask on every KV block of the inner loop, whereas ours masks only diagonal and boundary blocks. As \Cref{tab:kernel-micro} shows, the fused kernel produces both windowed outputs in less time than a single FlashAttention-2 SWA call, at every batch size.

\paragraph{LSE-combination variant.}
When maintaining a custom kernel is undesirable, the same saving is available from stock kernels: attention over disjoint key sets can be merged exactly through the log-sum-exp that FlashAttention already tracks, $o = e^{L_1-L}o_1 + e^{L_2-L}o_2$ with $L=\log(e^{L_1}+e^{L_2})$. We obtain the student output from a $\swindow$-window call and the teacher output by merging it with a shifted ``far-band'' call over the teacher-only keys beyond the student window, $\mathrm{Attn}(q_{\,\swindow:},\,k_{:-\swindow},\,v_{:-\swindow})$ with window $\lwindow-\swindow$, using a one-pass merge kernel. The student output is bitwise identical to the stock kernel by construction. On FlashAttention-3, whose SWA path does not suffer from the per-block masking issue, this variant computes both outputs in roughly $30\%$ less time than two FA3 calls (\Cref{tab:kernel-micro}).

\begin{table}[h]
\centering
\small
\setlength{\tabcolsep}{6pt}
\begin{tabular}{crrrrrr}
\toprule
\textbf{Batch} & \textbf{FA2 $1\times$} & \textbf{FA2 $2\times$} & \textbf{FA3 $1\times$} & \textbf{FA3 $2\times$} & \textbf{Fused (ours)} & \textbf{LSE (ours, FA3)} \\
\midrule
1 & 13.2 & 20.3 & 5.7 & 9.9 & 11.3 & 6.6 \\
4 & 56.1 & 84.4 & 24.8 & 40.5 & 45.0 & 30.7 \\
\bottomrule
\end{tabular}
\vspace{0.15cm}
\caption{Attention-only prefill time (ms) at 64K context with $\lwindow=8192$, $\swindow=4096$ (Qwen3-0.6B shape: 16 query / 8 KV heads, head dimension 128, bf16, H100). ``$1\times$'' is a single teacher-window call; ``$2\times$'' is the naive teacher-plus-student pair. Our fused kernel produces \emph{both} outputs faster than a single FlashAttention-2 call; the LSE variant is the fastest option on the FlashAttention-3 stack. Median over 10 timed iterations after warmup.}
\label{tab:kernel-micro}
\end{table}

\paragraph{End-to-end prefill throughput.}
\Cref{fig:prefill-fa3} in the main text and \Cref{fig:prefill-fa2} report end-to-end prefill throughput of the continually pre-trained Qwen3-0.6B configuration (28 layers, 4 TTT layers) against the SWA Transformer and IP-TTT baselines, at 64K/batch~1, 64K/batch~4, and 256K/batch~1, sweeping the teacher window with $\swindow=\lwindow/2$. Every cell is the mean of four independent repeats with all three models measured back-to-back on the same GPU. On the FlashAttention-2 stack (all models use FA2 for their standard attention; \ipttcd uses the fused dual-window and convolution kernels on its TTT layers), \ipttcd is within $-1.3\%$ to $+3.2\%$ of IP-TTT---at the largest window (16K) it is slightly \emph{faster}, because the fused kernel more than repays the dual-window cost. Relative to the naive two-kernel implementation, the fused kernels shorten end-to-end prefill by $10\%$--$13\%$ across windows ($11\%$--$14\%$ higher throughput). On the FlashAttention-3 stack (all models on FA3; \ipttcd uses the LSE variant plus the convolution kernel), throughput of every model rises by $25\%$--$90\%$ (more at larger windows) and the residual \ipttcd overhead is $+1.3\%$ to $+6.6\%$; at this point attention is no longer the bottleneck and the remaining cost is the window-independent student MLP path (student activations, convolutions, and the $\Delta\W$ scan). A CUTLASS-level port of the dual-window design to FlashAttention-3 would recover most of this remainder; we leave it to future work.

\begin{figure}[h]
    \centering
    \includegraphics[width=\linewidth]{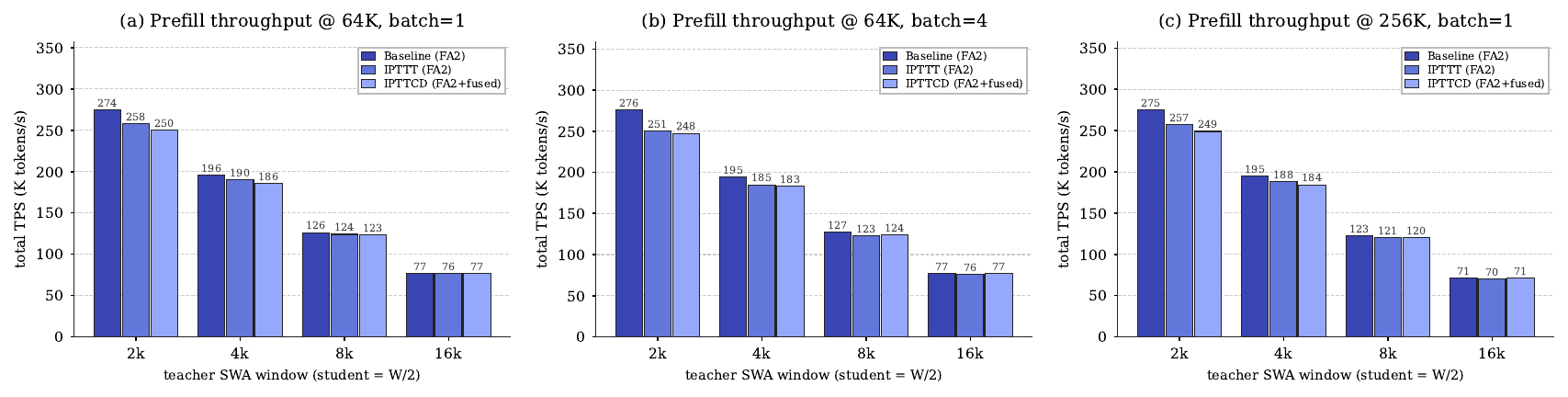}
    \caption{\textbf{End-to-end prefill throughput on the FlashAttention-2 stack.} Mean of four repeats on H100; \ipttcd uses the fused dual-window and convolution kernels on its TTT layers. Across all settings the gap to IP-TTT is $-1.3\%$ to $+3.2\%$.}
    \label{fig:prefill-fa2}
\end{figure}

\paragraph{Numerical equivalence.}
Because \ipttcd's fast-weight recursion can amplify kernel-level rounding differences, we validate the kernels at three levels on the continually pre-trained checkpoint. (i) \emph{Kernel level}: as reported above, the fused kernel matches FlashAttention-2 bit-for-bit on $99.7$--$99.9\%$ of output elements; the convolution kernel and the LSE student path are exactly bitwise. (ii) \emph{Distribution level}: over 10M tokens of held-out book text at 64K, the per-token KL divergence between the two implementations' next-token distributions has median $7.8\times10^{-4}$ nats (p99 $=4.0\times10^{-3}$), total-variation distance has median $1.25\%$ (p99 $=3.8\%$), and greedy top-1 decisions flip on $1.76\%$ of positions. (iii) \emph{Task level}: RULER NIAH accuracy is identical in all 8 evaluated (task, context) cells (four NIAH tasks at 32K and 64K). We caution that for such recursive architectures, kernel changes should be signed off with task metrics rather than logit comparisons.

\section{Theoretical Analysis}
\subsection{Proof of \Cref{thm:main}}
We first restate the proposition.
\begin{proposition}[Context distillation transfers an in-context solution]
\label{appen_thm:main}
    Under the setup and assumptions above, for any learning rate $\eta>0$, the update induced by preceding chunks satisfies
\begin{equation}
\begin{aligned}
    {y}_{\querypos}^{(T)\top}\Delta {y}_{\newquerypos}^{(S)}
    \ge
    \eta\,\hgap\,\zalign.
\end{aligned}
\end{equation}
Equivalently, the chunk update moves the later student hidden state at $q_2$ toward the teacher hidden state that already encodes the in-context solution for $q_1$.
\end{proposition}

\emph{Interpretation.} The vector $\bm{Y}_{\querypos}^{(T)}$ is the representation that solves $q_1$ using the long context. Equation~\eqref{eq:theory-main} shows that the distillation update stores the teacher-student residual from $q_1$ in the fast weights and replays it at a later aligned query $q_2$. Thus, the fast weights can carry over an in-context learning result even after the original supporting context is no longer available to the short-window student.

\begin{proof}
By Equations~\eqref{eq:theory-update} and~\eqref{eq:theory-student-plus},
\begin{align*}
     y_{\querypos}^{(T)\top}\Delta y_{\newquerypos}^{(S)}
    &= y_{\querypos}^{(T)\top} \Delta \W_{<j} z^{(S)}_{\newquerypos} \\
    &= \eta\, y_{\querypos}^{(T)\top} \sum_{t \in \mathcal{I}_{<j}} \bigl(y_t^{(T)}-y_t^{(S)}\bigr)z_t^{(S)\top}z^{(S)}_{\newquerypos}\\
    &= \eta\, y_{\querypos}^{(T)\top}\bigl(y_{\querypos}^{(T)}-y_{\querypos}^{(S)}\bigr)z_{\querypos}^{(S)\top}z^{(S)}_{\newquerypos},
\end{align*}
where the last equality follows from the assumptions. Applying the positive teacher-student gap and the related-query alignment gives
\[
    y_{\querypos}^{(T)\top}\Delta y_{\newquerypos}^{(S)}
    \ge
    \eta\,\hgap\,\zalign,
\]
which proves Equation~\eqref{eq:theory-main}.
\end{proof}




\end{document}